\RequirePackage{amsmath}
\documentclass[12pt]{llncs}

\usepackage[utf8]{inputenc}
\usepackage[T1]{fontenc}
\usepackage{latexsym}
\usepackage{amssymb}
\usepackage{lmodern}
\usepackage{url}
\usepackage{array}
\usepackage{geometry}
\usepackage{graphicx}
\usepackage{tabularx}
\usepackage{array}
\usepackage{mdwmath}
\usepackage{mdwtab}
\usepackage{float}
\usepackage{tikz}
\usetikzlibrary{shapes,arrows}
\usepackage{makecell}
\usepackage{diagbox}
\usepackage{booktabs}

\begin{document}

\setcounter{page}{1} 
\title{Rough Randomness and its Application}
\author{\textsf{A Mani}}
\institute{Machine Intelligence Unit, Indian Statistical Institute, Kolkata\\
203, B. T. Road, Kolkata-700108, India\\
Email: \texttt{$a.mani.cms@gmail.com$} \texttt{$amani.rough@isical.ac.in$}\\
Homepage: \url{https://www.logicamani.in}\\
Orcid: \url{https://orcid.org/0000-0002-0880-1035} }
\date{Received: 21st March 2023}

\maketitle

\begin{abstract}
A number of generalizations of stochastic and information-theoretic randomness are known in the literature. However, they are not compatible with handling meaning in vague and dynamic contexts of rough reasoning (and therefore explainable artificial intelligence and machine learning). In this research new concepts of rough randomness that are neither stochastic nor based on properties of strings are introduced by the present author. Her concepts are intended to capture a wide variety of rough processes (applicable to both static and dynamic data), construct related models, and explore the validity of other machine learning algorithms. The last mentioned is restricted to soft/hard clustering algorithms in this paper. Two new computationally efficient algebraically-justified algorithms for soft and hard cluster validation that involve rough random functions are additionally proposed in this research. A class of rough random functions termed large-minded reasoners have a central role in these.

\end{abstract}


\section{Introduction}

Rough sets is a formal approach to model vagueness, approximate reasoning, and uncertainty that has progressed tremendously over the last four decades from both pure and application perspectives in several directions \cite{zpb,ppm2,thrive2017,amedit}. In the axiomatic approach to granular computing \cite{am5586,am240}, due to the present author, approximations are generated as terms formed from granules. Subjective probabilistic and evidence theory based interpretations/analogies of specific versions of rough sets are known in the literature \cite{sdt05,yy2008,am9411,act2022}. However, these require very restrictive assumptions, and are of unclear ontology. 

A phenomenon is \emph{stochastically random} if it has probabilistic regularity in the absence of other types of regularity \cite{ank1986}. On the other hand, a sequence is \emph{algorithmically random} if and only if no computational agent recognizes it as possessing some rare property (that is properties valid over a set of measure zero that could be tested in a sufficiently effective way) \cite{ts2022,ank1986}. These ideas are realized in the contexts of subjective probability, and computability in several ways. Stochastic ideas of rough randomness that basically generalize measure-theoretic probability are known in the literature \cite{bliu2004}. However, they are essentially a hybrid approach to measure-theoretic probability, and are not applicable for the purposes of explainable AI and cluster validation through models. In fact, it is known in the psychology literature that humans cannot perceive ideas of stochastic randomness and weakenings thereof in real life to the point that they are very bad at it \cite{lbg1994}. This further suggests that the connections in the rough set literature between specific versions of rough sets and subjective probability theories (Bayesian or frequentist) are not properly grounded or mutually inconsistent. It is because rough reasoning, in the contexts, originate in some comprehension of attributes (their relation with the approximated object in terms of number or relative quantity and quality) \cite{amedit,ppm2}.
 
The central idea of stochastic randomness mentioned above can be generalized or exported research as follows: \emph{A phenomenon is random relative to a method of interpretation $\mathfrak{X}$ if it does have some regularity relative to $\mathfrak{X}$, and is otherwise apparently random}. Such a generalization needs to be consistent, novel, meaningful, and reasonable processes and predicates should be derivable in the application contexts of $\mathfrak{X}$. In this research, new concepts of rough randomness are introduced over classical, relation-based, cover-based, and general granular axiomatic rough sets  \cite{am5586,am240} in the light of the above by the present author. Related concepts of rough random functions are shown to be useful for systematically representing key aspects of soft and hard cluster validation, and dynamic information systems by her. These functions capture a number of associations of general rough approximations with rough objects, measures relating to semantic features, and thus provide an improved approach to approximate (including three-way) decision-making. A system of new algorithms to cluster validation that radically improve the computational complexity are invented in this research. These algorithms involve rough random functions that are used to construct a concept of a \emph{large-minded reasoner}. In fact, the latter is a rough random function that is mathematically sound.

\section{Background}
Information tables are representations of structured data in tabular form. They are additionally referred to as descriptive or knowledge representation systems in the artificial intelligence and machine-learning (AIML) literature. In general rough sets, such tables are not absolutely essential, however, they are obviously useful (if available). For details the reader is referred to \cite{am5559,am501,am1113}. 

An \emph{information table} $\mathcal{I}$, is a relational system of the form \[\mathcal{I}\,=\, \left\langle \mathfrak{O},\, \mathbb{A},\, \{V_{a} :\, a\in \mathbb{A}\},\, \{f_{a} :\, a\in \mathbb{A}\}  \right\rangle \]
with $\mathfrak{O}$, $\mathbb{A}$ and $V_{a}$ being respectively sets of \emph{objects}, \emph{attributes} and \emph{values} respectively.
$f_a \,:\, \mathfrak{O} \longmapsto \wp (V_{a})$ being the valuation function associated with attribute $a\in \mathbb{A}$. Values may additionally be denoted by the binary function $\nu : \mathbb{A} \times \mathfrak{O} \longmapsto \wp{(V)} $ defined by for any $a\in \mathbb{A}$ and $x\in \mathfrak{O}$, $\nu(a, x) = f_a (x)$.

An information table is \emph{deterministic} (or complete) if
\[(\forall a\in At)(\forall x\in \mathfrak{O}) f_a (x) \text{ is a singleton}.\] It is said to be \emph{indeterministic} (or incomplete) if it is not deterministic that is
\[(\exists a\in At)(\exists x\in \mathfrak{O}) f_a (x) \text{ is not a singleton}.\]

Relations may be derived from information tables by way of conditions in the following form: For $x,\, w\,\in\, \mathfrak{O} $ and $B\,\subseteq\, \mathbb{A} $, \[ \sigma xw  \text{ if and only if } (\mathbf{Q} a, b\in B)\, \Phi(\nu(a,\,x),\, \nu (b,\, w),) \] for some quantifier $\mathbf{Q}$ and formula $\Phi$. The relational system $S = \left\langle \underline{S}, \sigma \right\rangle$ (with $\underline{S} = \mathfrak{O}$) is said to be a \emph{general approximation space}. It should be noted that this universal feature of defining relations in general approximation spaces do not hold always in human reasoning contexts. Typically, information tables are finite. So it will be assumed that $card(\mathfrak{O}) = r < \infty $, and $card(\mathbb{At}) = n < \infty$.

In most applications, the neighborhood granule of a point $a \in S$ is taken to be $n(a) = \{b : Rba\}$ (or its inverse) because the maximal collections of mutually $\sigma$-related elements (if defined) is/was harder to compute. With recent improvements in accelerators like the GPU, scalability has considerably improved.   

\subsection{Distance Functions}
A \emph{distance function} on a set $S$ is a function $\rho : S^2 \longmapsto \Re_+$ that satisfies 
\begin{equation}
 (\forall a) \rho(a,a) = 0 \tag{distance}
\end{equation}
The collection $\mathcal{B} = \{B_\rho (x, r): x\in S \& r>0 \}$ of all $r$-spheres generated by $\rho$ is a weak base for the topology $\tau_\rho$ defined by 
\[V\in \tau_\rho \text{ if and only if } (\forall x \in V \exists r >0) B_\rho (x, r)\subseteq V\]

Any $\epsilon > 0$ and a distance function $\rho$ determines a tolerance $T$ defined by \[Tab \text{ if and only if } \rho (a, b) + \rho(b, a) \leq \epsilon.\] One can define other
tolerances through conditions such as 
\[\dfrac{\rho(a,b)+\rho(b, a)}{1+ \rho(a, b)+\rho(b, a)} \leq \epsilon. \]
The point is that a function much weaker than a semimetric suffices for defining a tolerance relation. More complex definitions are often possible. 

\begin{proposition}
For a numeric complete information table $\mathcal{I}$, the following holds:
\begin{enumerate}
\item {Valuations for each attribute are totally ordered by $\leq$,}
\item {$\mathcal{O}$ is totally ordered relative to the induced lexicographic order. }
\item {$\mathcal{O}$ is lattice ordered relative to $\leq$ defined by $(a_1, \ldots, a_n)\preceq (b_1, \ldots , b_n)$ if and only if $\&_i a_i\leq b_i$ with $a_i, b_i \in Ran (\nu(,At_i)$.}
\end{enumerate}
\end{proposition}
However, a numeric table is not necessary for any of the three properties to hold.

\subsection{Tolerance Relations}

Some familiarity with the algebraic theory of tolerances, and its use in general rough sets \cite{chtol1991,am501,am5019} will be assumed.
If $T$ is a tolerance on a set $S$, then a \emph{pre-block} of $T$ is a subset $K\subseteq S$ that satisfies $K^2 \subseteq T$. The set of all pre-blocks of $T$ is denoted by $p\mathcal{B}(T)$. Maximal pre-blocks of T with respect to the inclusion order are referred to as \emph{blocks}. The set of all blocks of $T$ is denoted by $\mathcal{B}(T)$. If $S = \left\langle underline{S},\,f_{1},\,f_{2},\,\ldots ,\, f_{n}, (r_{1},\,\ldots ,\,r_{n} )\right\rangle$ ($\underline{S}$ being a set and $f_i$ being $r_i$-place operation symbols interpreted on it) is an algebra, then a tolerance $T$ is said to be \emph{compatible} if and only if for each $i\in \{1, 2, \ldots n < \infty \}$,
\[ \&_{j=1}^{r_i} Ta_jb_j \longrightarrow Tf_i(a_1,a_2,\ldots a_{r_i}) f_i(b_1,b_2,\ldots b_{r_i}). \]  

When $S$ is a lattice, every tolerance is the image of a congruence by a surjective morphism $: S\longmapsto S$. Further, if $A, B\in \mathcal{B}(T)$, then $\{a\vee b : a\in A \& b\in B\}, \{a\wedge b : a\in A \& b\in B\} \in p\mathcal{B}(T)$. The smallest blocks containing these are unique, and the resulting lattice of blocks is denoted by $S|T$. The set $\mathbf{UBD}(S) = \{\mathcal{B}(T): \, T\in Tol(S)\}$ will be referred to as the \emph{universal block distribution} (UBD) of $S$. It can be assigned the same algebraic lattice order on $Tol(S)$. 

A sublattice $Z$ of a lattice $S$ is called a \emph{convex sublattice} if and only if it satisfies $(\forall{x,b}\in{Z})(x\leq{a}\leq{b}\longrightarrow {a}\in{Z})$. The blocks of a lattice are all convex sublattices. If $C$ is a subset of $S$ then $\downarrow{C}$, and $\uparrow{C}$ will respectively denote the lattice-ideal and lattice-filter generated by $C$. The following result \cite{hjb1982,gclk1983,chtol1991} is not usable for a direct computational strategy:
\begin{theorem}\label{fintol}
For a finite lattice $L$, a collection $\mathcal{C}$ of nonempty subsets is the set of all blocks of a tolerance $T\in Tol(L)$  if and only if it is a collection of intervals of the form $\{[a_i, b_i]: i\in I\} $, and 
\begin{itemize}
\item {$\bigcup_{i\in I} [a_i, b_i] = L$}
\item {For all $i, j\in I$, $(a_i = a_j \longrightarrow b_i = b_j)$. }
\item {$(\forall i, j\in I)(\exists k\in I)\, a_k= a_i\vee a_j \,\&\, b_i\vee b_j \leq b_k  $.}
\end{itemize}
\end{theorem}
\begin{theorem}
In the context of Theorem \ref{fintol},  
\begin{enumerate}
\item{$(\forall{C,E}\in{\mathcal{C}})\left(\downarrow{C}=\downarrow{E}\Longleftrightarrow{\uparrow{C}=\uparrow{E}}\right).$}
\item{For any two elements ${C,A}\in{\mathcal{C}}$ there exist ${E,F}$  such that $\left(\downarrow{C}\vee\downarrow{A}\right)\,=\,\downarrow{E},$\\ $(\uparrow{C}\vee\uparrow{A})\leq\uparrow{E},$ $\downarrow{F}\,\leq\,(\downarrow{A}\wedge\downarrow{C}),$ and $(\uparrow{C}\wedge\uparrow{A})\,=\,\uparrow{F}).$}
\end{enumerate}
\end{theorem}

For finite chains, the following can be said \cite{aji2016}
\begin{theorem}\label{chaintol}
\begin{enumerate}
 \item {
A collection $\mathcal{C}$ of subsets of the chain\\ $L_n = \left\langle\{0, 1, 2, \ldots n-1\}, \leq \right\rangle$ is the set of all blocks of a tolerance $T\in Tol(L)$ if and only if $\mathcal{C}$ is of the form $\{[n_i, m_i]: i=1, \ldots k \}$ for some $1\leq k \leq n-1$, with $n_1 = 0$, $m_k = n-1$, and $n_i < n_{i+1} \leq m_i +1$, and $m_i < m_{i+1}$ for all $i=1, \ldots k$.}
 \item {
A collection $\mathcal{C}$ of subsets of the chain $L_n = \left\langle\{0, 1, 2, \ldots n-1\}, \leq \right\rangle$ is the set of all blocks of a glued tolerance $T\in Glu(L)$ if and only if $\mathcal{C}$ is of the form $\{[n_i, m_i]: i=1, \ldots k \}$ for some $1\leq k \leq n-1$, with $n_1 = 0$, $m_k = n-1$, and $n_i < n_{i+1} \leq m_i < m_i +1$, and $m_i < m_{i+1}$ for all $i=1, \ldots k$.}
 \item {A collection $\mathcal{C}$ of subsets of the chain $L_n = \left\langle\{0, 1, 2, \ldots n-1\}, \leq \right\rangle$ is the set of all blocks of a congruence $R\in Con(L)$ if and only if $\mathcal{C}$ is of the form $\{[n_i, m_i]: i=1, \ldots k \}$ for some $1\leq k \leq n-1$, with $n_1 = 0$, $m_k = n-1$, and $n_i < n_{i+1} = m_i +1$, and $m_i < m_{i+1}$ for all $i=1, \ldots k$.}
\end{enumerate}
\end{theorem}

\emph{In this research, prefix or Polish notation is uniformly preferred for relations and functions defined on a set. So instances of a relation $\sigma$ are denoted by $\sigma a b$ instead of $a \sigma b$ or $(a, b) \in \sigma$. If-then relations (or logical implications) in a model are written in infix form with $\longrightarrow$.} 

\section{Rough Randomness}

The meta principle stated in the introduction when specialized to general rough sets has the following form: \emph{A phenomenon is roughly random if it can be modeled by general rough sets or a derived process thereof}. From a more concrete perspective, it should involve roughly random functions or predicates in some sense. On the basis of potential application in theoretical and applied studies, this is made precise below. 

\begin{definition}\label{roran1}
Let $\mathcal{A}_\tau$ be a collection of approximations of type $\tau$, and $E$ a collection of rough objects defined on the same universe $S$, then by a \emph{rough random function of type-1} (\textsf{RRF1}) will be meant a partial function \[\xi : \mathcal{A}_\tau \longmapsto E .\] 
\end{definition}

\begin{definition}\label{roran2}
Let $\mathcal{A}_\tau$ be a collection of approximations of type $\tau$, $\mathcal{S}$ a subset of $\wp(S)$, and $\Re$ the set of reals, then by a \emph{rough random function of type-2} (\textsf{RRF2}) will be meant a function \[\chi : \mathcal{A}_\tau \times \mathcal{S}\longmapsto \Re .\] 
\end{definition}

\begin{definition}\label{roran3}
Let $\mathcal{A}_\tau$ be a collection of approximations of type $\tau$, and $F$ a collection of objects defined on the same universe $S$, then by a \emph{rough random function of type-3} (\textsf{RRF3}) will be meant a function \[\mu : \mathcal{A}_\tau \longmapsto F .\] 
\end{definition}

\begin{definition}\label{hroran}
Let $\mathcal{O}_\tau$ be a collection of approximation operators of type $\tau_l$ or $\tau_u$, and $E$ a collection of rough objects defined on the same universe $S$, then by a \emph{rough random function of type-H} (\textsf{RRFH}) will be meant a partial function \[\xi : \mathcal{O}_\tau \times \wp(S) \longmapsto E .\] 
\end{definition}

It is obvious that a RRF1 and RRF2 are independent concepts, while a total RRF1 is an RRF3, and RRFH is distinct (though related to RRF3). The collection of all such functions will respectively be denoted by $RRF1(S, E, \tau)$,  $RRF2(S, \Re , \tau)$, $RRF3(S, F, \tau)$, and $RRFH(S, E, \tau)$.
The concepts are intended to capture a number of associations of general rough approximations with rough objects, and various measures relating to semantic features. Some examples are presented next. The use of a single universe with as opposed to multiple universes for handling temporal aspects is not a big issue.

\begin{remark}
These are not referred to as \emph{variables} by analogy with probability theory for technically correct reasons. Moreover, the analogy is very weak.   
\end{remark}

\subsection*{Examples: RRF}

A few classes of examples with roots in the practice of general rough sets are described in this subsection. 

\begin{example}\label{ex1}
Let $S$ be a set with a pair of lower ($l$) and upper ($u$)  approximations satisfying (for any $a, b, x \subseteq S$)
\begin{align}
x^l \subseteq x^u \tag{int-cl}\\
x^{ll} \subseteq x^l\tag{l-id}\\
a\subseteq b \longrightarrow a^l \subseteq b^l \tag{l-mo}\\
a\subseteq b \longrightarrow a^u \subseteq b^u \tag{u-mo}\\
\emptyset^l = \emptyset \tag{l-bot}\\
S^u = S\tag{u-top}
\end{align}
The above axioms are minimalist, and most general approaches satisfy them.

In addition, let 
\begin{align}
\mathcal{A}_\tau = \{x : (\exists a \subseteq S)\, x= a^l \text{ or }  x= a^u  \tag{1}\\
E_1 = \{(a^l, a^u):\, a\in S\} \tag{E1}\\
F = \{a: \, a\subseteq S \& \neg \exists b b^l = a \vee b^u = a\} \tag{E0}\\
E_2 = \{b: b^u = b \& b\subseteq S\}\tag{E2}\\
\xi_1 (a) = (a, b^u) \text{ for some } b\subseteq S\tag{xi1}\\
\xi_2 (a) = (b^l, a ) \text{ for some } b\subseteq S \tag{xi2}\\
\xi_3 (a) = (e, f) \in E_1 \,\&\, e= a \text{ or } f= a \tag{xi3}
\end{align}

$E_1$ in the above is a set of rough objects, and a number of algebraic models are associated with it \cite{am501}. A partial function $f: \mathcal{A}_\tau\longmapsto E_1$ that associates $a\in \mathcal{A}_\tau$ with a minimal element of $E_1$ that covers it in the inclusion order is a RRF of type 1. For general rough sets, this RRF can be used to define algebraic models and explore duality issues \cite{am5019}, and for many cases associated these are not investigated. A number of similar maps with value in understanding models \cite{amedit} can be defined. Rough objects are defined and interpreted in a number of other ways including $F$ or $E_2$.

Conditions \textsf{xi1-xi3} may additionally involve constraints on $b$, $e$ and $f$. For example, it can be required that there is no other lower or upper approximation included between the pair or that the second component is a minimal approximation covering the first. It is easy to see that 

\begin{proposition}
$\xi_i$ for $i=1, 2, 3$ are RRF of type-1. 
\end{proposition}
\end{example}

\begin{example}
In the context of the above example, rough inclusion functions, membership, and quality of approximation functions \cite{js09,ag2009} can be used to define RRF2s. An example is the function $\xi_5$ defined by 
\begin{equation}
 \xi_5(a, b) = \dfrac{Card(b\setminus a)}{Card(b)}
\end{equation}

In the algebraic models based on maximal antichains of definite objects due to the present author \cite{am9114}, it is useful to study the association of approximation with parts of maximal antichains. This is expressible with RRF of type H. This is considered in a separate paper.
\end{example}

\begin{proposition}\label{compp}
A rough random variable \cite{bliu2004} in the sense of Liu, is not a rough random function of any type.  
\end{proposition}
\begin{proof}
In the theory \cite{bliu2004}, a \emph{rough space} is a tuple $(S, \mathcal{F}, T \mu$ with $(S, \mathcal{F}, \mu$ being a measure space, and  $T\in \mathcal{F}$. A \emph{rough variable} $\pi$ is a measurable function from a rough space into the reals. If $B$ is a Borel subset of $R$ then $\{x \in S; \, \pi (x) \in B\} \in \mathcal{F}$.  The lower and upper approximations of $\pi$ are $\pi^l = \{\pi(x); x\in T\}$, and $\pi^u = \{\pi(x) ; x\in S\}$. That is a function is approximated by its restricted range -- this interpretation is not cleanly explained in the book \cite{bliu2004}. In the perspective of rough sets, some part of the range of $\pi$ is approximated, and not $\pi$.

The trust associated with an event $A$ is defined as (the mean of upper and lower trusts): $\tau(A) = 0.5 (\frac{\mu(A)}{\mu(S)} + \frac{\mu(A\cap T)}{\mu(T)}$. A \emph{rough random variable} (Liu) is a function $\chi$ from a probability space $(X,\mathcal{S}, p)$ to the set of rough variables such that $\tau \{\chi(x) \in B\}$ is a measurable function of $x$ for any Borel set $B$ of $R$. As probability spaces are not related to any of the domains of RRFs, the concepts are unrelated in general. 
\end{proof}

\section{Application: New Cluster Validation Algorithm}

Hard clusters are essentially partitions of a set that satisfy additional conditions on the partitions, while soft clusters are collections of pairs of subsets of the form $\{(C_i, E_i)\}$. $C_i$ is referred to as the core and $E_i$ as the exterior. $C_i$s are mutually disjoint, while $E_i $ is disjoint from $C_i$ for each $i$. Additional conditions are usually imposed \cite{adc2020,hmmr2016}. Cluster validation through rough sets is proposed by the present author in a recent paper \cite{am2021c}. Here the context of distance based soft/hard clustering is considered and new algorithms are invented. 

\subsection{Meta Algorithm-1}

This is a meta-algorithm as some steps that require high performance computing can be implemented in many ways. 

Suppose a hard clustering $\{C_i\}_{i=1}^k$ or a soft clustering $\{(C_i, E_i)\}_{i=1}^k$ \cite{adc2020,hmmr2016} obtained through any method is given.
\begin{description}
 \item [Distance]{Specify distinct distance functions by each column (attribute) or between objects that are meaningful.}
 \item [Similarity]{Define a similarity (tolerance relation) for each column or between objects.}
 \item [Combination]{Combine to a single tolerance relation over objects on the table }
 \item [Similarity Matrix]{Compute the similarity matrix (several HPC methods are possible). }
 \item [Blocks]{Compute the blocks of the tolerance by a maximal clique algorithm (for example the modified Bronkerbosch algorithm   \cite{eppstein2010}).}
 \item [Approximations]{Compute granular rough approximations of $C_i$, and $E_i$ for each $i$ and estimate the closeness of the cluster core or exterior to decide on validation.}
\end{description}

If the rough model can explain the soft/hard clustering, then the latter is meaningful and valid. 

\subsection{Axiomatic Granular Reversed Similarity Based Semi-Supervised Algorithm (AGRSSA)}

This new algorithm requires a total order on each column (attribute), and an order-compatible metric over objects or one for each column that is compatible with the order. Specifically, it applies to all numeric (real valued) datasets. The essential steps are

\begin{description}
 \item [Distance]{Specify distinct distance functions by each column (attribute) or between objects. }
 \item [Exploratory Statistics]{Identify $q$-quantiles at a suitable level of precision. Let these be $\{q
 \_1,\, q_2, \, \ldots q_{f}\}$ based on the distance specified earlier.}
 \item [Interval Boundaries]{Interval boundaries can be computed through the sequence $\bot, q_1-e_1, q_1+e_1, q_2-e_2, q_2+e_2, \, \ldots ,q_f-e_f, q_f+ e_f, \top  $. The quantities $e_1, e_2, \ldots e_f$ being determined as a fraction of the standard deviation, a local standard deviation or other local measures of variation.}
 \item [Blocks]{Specify blocks as the intervals according to the minimal scheme or the exhaustive/selective tolerance discovery algorithm (specified below). The subsequent steps in the latter are different }
  \item [Approximations]{Compute granular rough approximations and perform decision-making. If a set of objects $H$ are to be approximated, then
  \begin{enumerate}
   \item {The lower approximation of $H$ is the union of blocks included in it.}
   \item {The lower approximation of $H$ is the union of blocks that have some common elements with $H$.}
  \end{enumerate}
}
 \item [Ontology]{Determine tolerance relation from blocks, and specify ontology.}
\end{description}

It is possible to get the values $e_i$ through exploratory statistical methods alone. However, this involves decisions based on a good understanding of the attribute values. 

The specification of blocks should ideally be a supervised step as it relates to an understanding of the relation of directional increase or decrease of columns. These can optionally be deduced based on the decision and label attributes, in which case, the algorithms become purely data-driven.   

Let the set of all permutations of the set $\prod_{i=1}^{n} \{1, \ldots, k_i\}$ be denoted by\\ $\Sigma(n, k_1, \ldots k_n)$

\subsection*{Scheme: Minimal (AGRSSA-M)}
For simplicity, suppose that there is exactly one column for decisions and no additional decision/label columns are present. 
The necessary steps for constructing the blocks from the possible intervals of each conditional attribute are as follows:
\begin{enumerate}
 \item {From the q-quantiles form intervals with possibly non-empty intersection, and subject to the constraint specified in Theorem \ref{chaintol}. }
 \item {Determine the interval boundaries for each conditional attribute by analyzing the gross changes in decisions (this is a heuristic). This may additionally lead to a good choice of a permutation $\sigma \in \Sigma(n, k_1, \ldots k_n)$. It being assumed that the $i$th attribute has $k_i$ blocks (for $i= 1, 2, \ldots n$).}
 \item {Form the blocks of the conditional attributes based on the interval boundaries of each conditional attribute,\\ and relative to the permutation operator $\sigma$.}
\end{enumerate}
In this case, additional exploratory statistical methods can be very useful in determining the choice of $\sigma$.

\subsection*{Exhaustive Tolerance Discovery Algorithm (AGRSSA-LMR)}

The mathematical concept of a large-minded reasoner is introduced because of its role in the exhaustive tolerance discovery algorithm invented below.

\begin{definition}\label{lmr}
By a \emph{large-minded reasoner} will be meant a partial function $\mathbf{\psi}: \mathbf{UB}(A_1)\times \mathbf{UB}(A_2)\times \ldots \mathbf{UB}(A_n) \longmapsto \mathbf{UB}(A) $. 
\end{definition}
Basically, it is intended to specify the possibly reasonable collections of blocks in a situation, with the assumption that the collection is the result of eliminating those not in $dom(\mathbf{\psi})$.

\begin{enumerate}
 \item {From the q-quantiles form intervals with possibly non-empty intersection, and subject to the constraint specified in Theorem \ref{chaintol}. }
 \item {Eliminate unreasonable interval boundaries for each conditional attribute by analyzing the gross changes in decisions (this is a heuristic). }
 \item {Specify the large minded reasoner $\mathbf{\psi}$}
 \item {Compute relevant approximations and decision regions for every defined instance of $\mathbf{\psi}$.}
 \item {Select the most or optimally appropriate instance(s) of $\mathbf{\psi}$.}
 \item {Explain the data context on the basis of the associated tolerance(s).}
\end{enumerate}

\begin{definition}
By an \emph{interpreted large-minded reasoner} associated with $\psi$ of Def. \ref{lmr} will be meant a partial function $\mathbf{\psi}^*: \mathbf{UB}(A_1)\times \mathbf{UB}(A_2)\times \ldots \mathbf{UB}(A_n)\times \wp(S) \longmapsto \mathbf{UB}(A) $, that indicates the granular components or parts of approximations of subsets. 
\end{definition}

\textbf{Explanation}: For example, if for a subset $F$, $F^l$ is the union of the blocks $\{ B_i\}$, then $\mathbf{\psi}^*$ says how the $B_i$s are formed from blocks of chains. Since more complex definitions of approximations are admissible (see \cite{am240,am501}), \emph{granular components or parts of approximations} are referred to. Term functions generated from basic set-theoretic operations can be used to write the definition of $\psi^*$ explicitly. However, the explicit definitions are not used in this paper. 

\begin{theorem}
Both algorithms AGRSSA-M and AGRSSA-LMR are well-defined, and they compute the intended tolerance(s).  
\end{theorem}

\begin{proof}
AGRSSA-M computes the intended tolerance because Theorem \ref{chaintol}, and the steps of the algorithm ensure that tolerances are initially defined on each column (attribute). 
If a lattice $L$ is a direct product $\prod_{i=1}^{n}L_i$ of the lattices $L_i$, then $Tol(L) \equiv Tol(L_1)\times Tol(L_2) \times \ldots Tol(L_n) $ (as tolerances on a product of finite lattices are directly decomposable \cite{jn1982}). This ensures that AGRSSA-M yields a tolerance.

The same results ensure that AGRSSA-LMR yields a subset of the tolerance lattice.
\end{proof}

It is easy to see that 

\begin{proposition}
$\mathbf{\psi}^*$ is a \textsf{RRF} of type H. 
\end{proposition}

\subsection{Adaptive/Dynamic Information Tables}

Both adaptive and dynamic information tables, that model change or correspondence, are investigated in the literature with mostly computational goals \cite{zspo05,hth2013,ytcjh2016,ssko2021} in mind. Changes in knowledge representation, and meaning are considered through correspondences \cite{am3600} and models \cite{am240,am501} in the present author's work. In these contexts, the axiomatic granules at each stage may be transformed in subsequent stages, and then it may be necessary to speak of correspondences between granules too.

If a dynamic process can be modeled by a finite sequence of information tables, then the tables at each step are likely to differ from the table in the preceding step. These changes may occur in following ways:
\begin{enumerate}
 \item{Objects are added ($O+$), deleted ($O-$) or both added and deleted ($O\pm$),} 
 \item{attributes are added or deleted ($At+$, $At-$ or $At\pm$), and}
 \item{the values associated with object-attribute pairs are modified ($V+$).}
\end{enumerate}

Correspondences between information tables may reflect changes of the above type. A finite sequence of information tables can also be written as a 3rd order tensor with tensor-dimensions objects, attributes and timestamps. Note that for each timestamp, the projection on the first two tensor-dimensions can be an indeterministic information table. In most ML literature, the term \emph{dimension} refers to the number of attributes. Therefore, the idea of a sequence of information tables may be easier for the reader.

Corresponding to each table, a finite number of lower and upper approximations may be associated. However, for simplicity, it suffices to take a sequence of two information tables with two pairs of corresponding lower and upper approximations. For this understanding, auxiliary maps between objects are essential for keeping track of changes.

Arguably, the easiest representation would be to combine all information tables into one and associate a finite sequence of partial lower and upper approximation maps. This corresponds with the intent in generalized granular operator spaces and variants studied by the present author in the papers \cite{am240,am501,am9402}.

In all the mentioned cases, one or more rough random functions of different types can be used to model the process or context.

\section{Remarks and Directions}

In AIML practice, mathematical objects and concepts are often badly approximated or estimated because of computational constraints, or the very algorithm (such as those in deep learning or swarm optimization) may not be tractable. A major problem is to invent fully explicable models with algorithms in most situations.  Cluster validation, in particular, has always been relative to arbitrary measures (or to ground truth subject to availability), except for recent work \cite{am2021c} that tries to compare with meaningful algebraic models. The present research is an extension of this approach. The best part is that the proposed algorithms match the mathematical intent of the computation.

All the three algorithms are intended to help in specifying at least one rough set model, and therefore a validation of the classification under consideration. The main advantage of the AGRSSA-LMR  are as follows:
\begin{enumerate}
 \item {Tolerances defined by relatively complex conditions can possibly be isolated. }
 \item {The algorithm is computationally light, and}
 \item {It is algebraically explicable.}
\end{enumerate}

In this research, new concepts of rough randomness of different types are invented and shown to be applicable in studying algebraic models, and computing the validity of soft and hard clusters. Applications to dynamic information tables, clustering, and other application contexts of general rough sets are naturally motivated. In future work, finer aspects of large minded reasoners, and a more detailed evaluation of rough randomness for meaning evolution in rough reasoning contexts and dynamic information tables from a granular perspective will be investigated. 

\begin{small}
\begin{flushleft}
\textbf{Acknowledgement:} This research of the present author is supported by woman scientist grant no. WOS-A/PM-22/2019 of the Department of Science and Technology.
\end{flushleft}
\end{small}
\bibliographystyle{splncs04.bst}
\bibliography{algroughnov2022+}

\end{document}